\documentclass[conference,onecolumn]{IEEEtran}
\IEEEoverridecommandlockouts
\usepackage{graphicx}%
\usepackage{multirow}%
\usepackage{amsmath,amssymb,amsfonts}%
\usepackage{amsthm}%
\usepackage{mathrsfs}%
\usepackage{xcolor}%
\usepackage{textcomp}%
\usepackage{manyfoot}%
\usepackage{booktabs}%
\usepackage{algorithm}%
\usepackage{algorithmicx}%
\usepackage{algpseudocode}%
\usepackage{listings}%
\usepackage{latexsym}
\usepackage{booktabs}
\usepackage{enumitem}
\usepackage{multirow,stfloats}
\usepackage{verbatim}
\usepackage{comment}
\usepackage{array}
\usepackage[caption=false,font=normalsize,labelfont=sf,textfont=sf]{subfig}
\usepackage{stfloats}
\usepackage{url}
\usepackage{cite}
\newcommand{\vect}[1]{\mathbf{#1}} 
\newcommand{\mat}[1]{\mathsf{#1}} 

\newcommand{\E}[1]{\textsf{E}\left[{#1}\right]}   
\newcommand{\Ep}[2]{\mathsf{E}_{#1\!}\left[{#2}\right]}   

\newcommand{\snorm}[1]{\left\| {#1} \right\|^2}

\newtheorem{theorem}{Theorem}
\newtheorem{definition}{Definition}[section]%

\begin{document}
\title{Heavy-Tailed Principal Component Analysis}

\author{
\IEEEauthorblockN{Mario Sayde$^{\dagger}$, Christopher Khater, Jihad Fahs$^{\dagger}$, Ibrahim Abou-Faycal$^{\dagger}$}
\IEEEauthorblockA{
Electrical and Computer Engineering Department\\
American University of Beirut\\
Riad El-Solh, Beirut 1107 2020, Lebanon\\
mas191@aub.edu.lb, ckk06@aub.edu.lb, jihad.fahs@aub.edu.lb, ibrahim.abou-faycal@aub.edu.lb\\
$^{\dagger}$ These authors contributed equally to this work.
}
}
\maketitle

\begin{abstract}
Principal Component Analysis (PCA) is a cornerstone of dimensionality reduction, yet its classical formulation relies critically on second-order moments and is therefore fragile in the presence of heavy-tailed data and impulsive noise. While numerous robust PCA variants have been proposed, most either assume finite variance, rely on sparsity-driven decompositions, or address robustness through surrogate loss functions without a unified treatment of infinite-variance models. In this paper, we study PCA for high-dimensional data generated according to a superstatistical dependent model of the form ${\bf X} = A^{\frac{1}{2}} {\bf G}$, where $A$ is a positive random scalar and ${\bf G}$ is a Gaussian vector. This framework captures a wide class of heavy-tailed distributions, including multivariate-t and sub-Gaussian $\alpha$-stable laws. We formulate PCA under a logarithmic loss, which remains well defined even when moments do not exist. Our main theoretical result shows that, under this loss, the principal components of the heavy-tailed observations coincide with the principal directions induced by the covariance matrix of the underlying Gaussian generator. Building on this insight, we propose robust estimators for this covariance matrix directly from heavy-tailed data and compare them with the empirical covariance and Tyler’s scatter estimator. Extensive experiments, including background denoising tasks, demonstrate that the proposed approach reliably recovers principal directions and significantly outperforms classical PCA in the presence of heavy-tailed and impulsive noise, while remaining competitive under Gaussian noise.
\end{abstract}


\section{Introduction}
\label{sec:stable}
Principal Component Analysis (PCA) is at the heart of modern data exploration. First described by \cite{pearson1901} and later systematized by \cite{hotelling1933}, it rotates the data so that the dominant directions of variability are determined, giving analysts a compact but accurate depiction of high-dimensional information. However, \emph{Classical} PCA is sensitive; even a small number of extreme values or corrupted data points can significantly distort the result \cite{10.1093/biomet/87.3.603}. 

To address this limitation, researchers have developed a variety of robust PCA approaches designed to maintain accuracy and stability, particularly when working with noisy or heavy-tailed data. A turning point was the formulation of \cite{candes2011robust}, who formulated the robust PCA as a convex program that decomposes a data matrix into a low-rank component that captures the underlying signal and a sparse component that absorbs gross errors. Because this separation is both principled and computationally tractable, Principal Component Pursuit (PCP) has become a reference method in image restoration and anomaly detection. Stable PCP establishes that low-rank and sparse components can still be reliably recovered in the presence of additional small entry-wise noise, with the reconstruction error bounded proportionally to the noise level \cite{zhou2010stable}.  Projection‑pursuit ideas coupled with robust covariance estimation define Robust Principal Component Analysis (ROBPCA), introduced by \cite{hubert2005robpca}. ROBPCA identifies directions that maximize a robust measure of spread such as the median absolute deviation (MAD), the $Q_n$ scale and the inter-quartile range, to name a few, and has been proven effective in chemometrics, environmental monitoring, and bioinformatics. Scaling the low‑rank–plus–sparse paradigm to modern data volumes, \cite{xu2010robust} proposed Robust Principal Component Pursuit (RPCP), an augmented Lagrangian algorithm with iterative thresholding that delivers the same clean separation while remaining efficient in high‑dimensional settings. The authors in~\cite{frahm2007tyler} replaced the sample covariance with a spectral estimator based on Tyler's M-estimator. This method works well with heavy-tailed elliptical data. In high-dimensional financial applications, this robust scatter matrix produces more dependable principal components by restoring the bulk eigenvalue distribution predicted by the Marčenko-Pastur law~\cite{frahm2007tyler,marvcenko1967distribution}.

Robustness can also be achieved by replacing the quadratic loss that underlies classical PCA. For example, \cite{he2011mcc} derived a rotationally invariant PCA under the maximum correntropy criterion (MCC), a similarity measure less sensitive to non‑Gaussian noise. Their half‑quadratic optimization scheme iteratively solves weighted least-squares problems, producing an estimator that maintains high breakdown performance without sacrificing computational practicality. Similarly, R1-PCA~\cite{ding2006r} and PCA-L1~\cite{kwak2008principal} are two variants based on the $L_1$ norm in the feature space, significantly improving the robustness against outliers. While R1-PCA is rotationally invariant, PCA-L1 is not, but is reliably converging to a local optimum and is found to improve performance in data sets contaminated with outliers. Complementing these advances, \cite{wei2017estimation} developed a median of means covariance estimator with non-asymptotic guarantees under the assumption of finite fourth moments. Such estimators can be plugged into eigen decomposition routines to obtain robust principal components when moment assumptions are weak or unknown. An online algorithm that updates the underlying subspace one observation at a time is introduced in~\cite{xiao2019online}, enabling real‑time video surveillance and equipment health monitoring. The work of~\cite{rahmani2017coherence} further accelerated subspace identification through coherence‑based screening, trimming both computation and memory requirements. The low‑rank–plus–sparse idea has been generalized to multi‑way data. In \cite{Lu_2016_CVPR}, the authors proposed Tensor Robust PCA (TRPCA), using tensor Singular Value Decomposition (t‑SVD) and its induced nuclear norm to recover exact low-rank tensors corrupted by sparse outliers. Subsequent refinements introduced alternative tensor norms~\cite{lu2019tensor} extending the applicability of robust PCA to hyperspectral imaging and neuroscience. More recently, to reduce sensitivity to outliers, a method based on applying standard PCA to a computed entropic covariance matrix capturing statistical similarity via divergence has been proposed in~\cite{Nakao2023EntropicPCA} and a Cauchy PCA that replaces the usual Gaussian log-likelihood formulation with a multivariate Cauchy likelihood is implemented in~\cite{Fayomi2024CauchyPCA} using an iterative algorithm.

Recent attention has shifted to data that may not even possess finite variances. In~\cite{https://doi.org/10.1111/anzs.12385}, the authors proposed a characteristic function transformation that embeds the data using a bounded non linear kernel as a special case of kernel PCA. This CT‑PCA inherits the robustness of kernel PCA yet tolerates heavy tails and infinite-variance distributions. The method is based on projections based on spheres \cite{liu2019spherical} and analytical bounds for PCA in high dimensions~\cite{blanchard2007statistical}.

Robust PCA methods now underline a wide range of tasks: background subtraction in video sequences \cite{candes2011robust,bouwmans2018applications}, denoising gene expression matrices \cite{hubert2005robpca}, anomaly detection in financial time series \cite{cont2002dynamics}, cleaning of environmental sensors \cite{serneels2008principal}, and extraction of vessels from coronary X-ray angiography \cite{qin2022robust}. Although some robust PCA variants remedy classical PCA’s vulnerability to heavy-tailed or power law distributions\footnote{A distribution $F(x)$ is said to have a right tail power law if $1 - F(x) = L(x) x^{-\alpha}$ for some $\alpha > 0$ and some slowly varying function $L(x)$. A similar definition can be applied to the left tail.}, several challenges still exist; for example, the problem of scaling to truly massive data and the performance under heavy tails having an infinite second moment and potentially infinite first moments. 
\subsubsection*{Our contributions}
In this paper, we address the latter challenge. More specifically, we consider high-dimensional data that are generated according to a  superstatistical dependent process of the form ${\bf X} = A^{\frac{1}{2}} {\bf G}$, a model of particular interest. In this construction, $A$ is a continuous positive random variable and $\vect{G}$ is a Gaussian vector. Such dependent multivariate models arise naturally in a wide range of scenarios. For example, random feature models have been employed to better understand complex behaviors of neural networks~\cite{Rahimi2007,Mei2019TheGE,Gerace_2021}. In~\cite{Adomaityte2023ClassificationOH}, the authors study high-dimensional classification problems by modeling data with heavy tails via a superstatistical framework~\cite{beck2003}. Related examples appear in~\cite{wainwright1999} and~\cite{Adomaityte2024HighDimRobustReg}. In particular, several well-known distributions can be expressed within this superstatistical setting, including the infinite-variance sub-Gaussian\footnote{In some communities, the term ``sub-Gaussian" refers to distribution functions whose tails are faster than those of a Gaussian. This is not the case in this manuscript.} $\alpha$-stable laws~\cite{sam1994} and the multivariate t-distribution~\cite{roth2013}, among others. More details are provided in Section~\ref{sec:pre}. 

We formulate a PCA problem where we assume that the observed high-dimensional data is due to a generator matrix and a low-dimensional set of features. Our goal is to reduce the dimensionality of the observed data matrix by only retaining the low-dimensional set of features. To account for a potential infinite-variance model of the data distribution, we impose a logarithmic fidelity criterion, one that we try to minimize by finding the ``best" low-dimensional vector of features under a linear transformation of the data matrix. Except for the work of~\cite{Xie2014CauchyPCA} that implemented a Maximum-Likelihood (ML) estimator for a low-rank Cauchy matrix under a logarithmic loss, this work represents, up to our knowledge, the first generic treatment of the PCA problem for potentially heavy-tailed, infinite variance distributions under a logarithmic loss function. Our contributions are threefold: 
\begin{itemize}
\item We consider statistically dependent data, potentially heavy-tailed with infinite variance, that are subordinated to a Gaussian vector ${\bf X} = A^{\frac{1}{2}} {\bf G}$. Under a ``logarithmic" cost, we prove that the principal components of the heavy-tailed data are determined by performing standard PCA on ${\bf G}$, via the eigen-decomposition of its covariance matrix $\Sigma$.
\item We propose and test alternative methods to estimate the covariance matrix $\Sigma$ of the vector ${\bf G}$ from the heavy-tailed data matrix $\mat{X}$. We compare our methods with the empirical covariance estimator and with Tyler's scatter estimator and show the superior performance, robustness, and consistency of our approach for estimating the principal directions for both Gaussian and heavy-tailed data. 
\item We provide experimental results for two different applications on background denoising, showing the robustness of the proposed heavy-tailed PCA compared to the standard PCA, and its enhanced suppression of heavy-tailed impulsive noise in addition to remaining adequate in the presence of standard Gaussian noise. The experimental results can be reproduced using the code publicly available in~\cite{aub_htp}.
\end{itemize}

The rest of this paper is organized as follows: In section~\ref{sec:pre} we introduce the superstatistical noise model and in section~\ref{sec:pf} we present the problem formulation. The main results are listed in section~\ref{sec:mr}. Section~\ref{sec:method} is dedicated to methods for estimating the covariance matrix of the underlying Gaussian vector from the heavy-tailed data, and in section~\ref{sec:ExpRes} we present applications of the proposed heavy-tailed PCA algorithm. Section~\ref{sec:conc} concludes the paper. 

\paragraph*{Notations}
We adopt the following notational conventions. The lowercase regular font symbols $(a, b, \cdots)$ represent scalar deterministic quantities, whereas the uppercase symbols $(A, X, \cdots)$ are scalar Random Variables (RV). The lowercase bold font characters $(\boldsymbol{\omega}, \vect{x}, \cdots)$ are reserved for deterministic vector quantities, and the uppercase counterparts $(\vect{A}, \vect{X}, \cdots)$ are random vectors, while $\mat{A}$ and $\mat{A}^\top$ denote a deterministic matrix and its transpose, respectively. We use the symbol $\mathbb{R}$ to denote the set of real numbers.

\section{Preliminaries}
\label{sec:pre}
A superstatistical dependent vector is defined as ${\bf X} = A^{\frac{1}{2}} {\bf G}$ where ${\bf X}$ denotes the random vector of interest, $A$ is a positive random variable with probability density function $p_A(a)$ --whenever it exists, and ${\bf G}$ is a zero-mean Gaussian vector with covariance matrix $\Sigma$. This construction provides a flexible framework for modeling complex multivariate statistical dependencies by introducing a random scaling factor $A^{\frac{1}{2}}$ on top of the Gaussian structure ($\vect{G}$). The presence of the multiplicative random variable can lead to heavy-tailed non-Gaussian behaviors, making such models valuable for applications where standard Gaussian assumptions fail to capture the observed variability.

The relevance of this construction can be seen in several areas. For example, in machine learning and neural networks, random feature models have been used to investigate and understand intricate behaviors such as generalization and robustness~\cite{Rahimi2007, Mei2019TheGE, Gerace_2021}. It falls under the family of ``elliptic-like" distributions considered in~\cite{Adomaityte2024HighDimRobustReg,Adomaityte2023ClassificationOH} as contamination models for the covariates, where in~\cite{Adomaityte2023ClassificationOH} for example, classification in high-dimensional settings is studied under the assumption of heavy-tailed data generated via a superstatistical mechanism, reflecting the growing recognition that real-world data often depart from Gaussianity~\cite{beck2003}. Other examples of the use of such models can be found in~\cite{wainwright1999}, further illustrating their breadth of applicability in statistical learning and signal processing.

From a theoretical perspective, the superstatistical framework unifies several important families of probability distributions. In particular, it encompasses well-known heavy-tailed models such as sub-Gaussian $\alpha$-stable laws~\cite[Appendix A]{sayde2025heavy} in which case $A \sim S\left(\frac{\alpha}{2},1,\cos(\frac{\alpha \pi}{4})^{\frac{2}{\alpha}},0\right)$ is a totally skewed $\alpha$-stable variable\footnote{An $\alpha$-stable variable has four parameters and is denoted $X \sim \mathcal{S}(\alpha, \beta, \gamma, \delta)$, where $0 < \alpha < 2$ is the characteristic exponent, $|\beta| \leq 1$ is the skewness, $\gamma > 0$ is the scale and $\delta \in \mathbb{R}$ is a location parameter. Whenever $\beta = \pm 1$, $X$ is called totally skewed. If in addition $0 < \alpha < 1$, $X$ is one-sided (positive if $\beta = 1$ and negative otherwise).}~\cite{sam1994} and the multivariate Student’s t-distribution with $\nu$ Degrees of Freedom (DoF) where $A = \frac{\nu}{U}$~\cite{roth2013} and $U$ is a chi-squared distributed variable with $\nu$-DoF. More examples can be found in~\cite[Table 1]{Adomaityte2024HighDimRobustReg}.

Both of these examples are extensively used in practice: $\alpha$-stable models capture impulsive and heavy-tailed noise behaviors in communications and finance, and play a central role in the theory of heavy-tailed density models being limiting distributions of a Generalized Central Limit Theorem (GCLT)~\cite{kolmo} in a similar way to what the Gaussian distribution represents for laws with a finite second moment. On the other hand, multivariate t-distributions serve as robust alternatives to Gaussian models in multivariate statistics. This shows that the construction is not only mathematically elegant, but also practically relevant, providing a unifying lens through which several known heavy-tailed models can be understood as special cases.

\section{Problem Formulation}
\label{sec:pf}

Let $\mat{X} \in \mathbb{R}^{d\times n}$ be the data matrix, where $d$ is the dimension of the features and $n$ represents the number of samples drawn independently of random vector ${\bf X} \in \mathbb{R}^d$. In the ``standard" $\text{L}^2$-PCA, one tries to find an $m$ ($< d$) dimensional vector space of new features by minimizing the cost function:
\begin{equation}
\label{eq:l2pca}
    C(\mat{W},\mat{V}) = \snorm{ \mat{X} - \mat{W}\mat{V} },
\end{equation}
where $\mat{W} \in \mathbb{R}^{d \times m}$ is the generator matrix such that $\mat{W}^\top \mat{W} = \mat{I}_{m}$ --the $m \times m$ identity matrix, and $\mat{V} \in \mathbb{R}^{m\times n}$ denotes the observation matrix of the low-dimensional set of features. According to the projection theorem \cite{luenberger1997optimization}, the minimizer ${\mat{V}}$ of \eqref{eq:l2pca} is $\mat{V} = \mat{W}^\top\mat{X}$ for any given $\mat{W}$. The minimum is then achieved whenever $\mat{W}$ is the matrix of eigenvectors corresponding to the largest $m$ eigenvalues of the covariance matrix $\Sigma$ of the underlying random vector ${\bf X} \in \mathbb{R}^d$. The empirical value of $\Sigma$ is calculated as $\Sigma = \frac{1}{n} \mat{X} \mat{X}^{\top}$. 

We consider this dimensionality reduction problem where we assume that the vector ${\bf X}$ is heavy-tailed and, more specifically, superstatistical of the form ${\bf X} = A^{\frac{1}{2}} {\bf G}$. 

Naturally, minimizing the cost function given by equation~\eqref{eq:l2pca} is non-sensible due to the potential non-existence of the covariance matrix of ${\bf X}$. Instead, in this work, we propose an alternative logarithmic average loss function defined as: 
\begin{equation}
    \Ep{{\bf X}}{ \ln \left( 1 + \snorm{ {\bf X} - \mat{W}{\bf V}} \right) }.
    \label{eq:logcost}
\end{equation}

Our goal is to find, under linear transformations, the ``best" low-dimensional set of features ${\bf V} = \mat{M}^\top {\bf X}$ and the best generator matrix $\mat{W}$, i.e. ones that minimize~\eqref{eq:logcost}.
We conduct this analysis in two steps: Our first result states that the minimizer ${\bf V}$ of~\eqref{eq:logcost} is ${\bf V} = \mat{W}^\top{\bf X}$ for any given $\mat{W}$ such that $\mat{W}^\top \mat{W} = \mat{I}_{m}$. Said differently, given a generator matrix $\mat{W} \in \mathbb{R}^{d \times m}$, we have

\begin{equation}
    \mat{W} = \underset{\mat{M} \in \mathbb{R}^{d \times m}}{\arg\min} \,\,\Ep{X}{\ln \left( 1 + \snorm{ {\bf X} - \mat{W}\mat{M}^{\top}{\bf X} }\right)}.
    \label{eq:linopt}
\end{equation}
Moreover, we show that the best $m$ dimensional vector of features ${\bf V}$ that minimizes~\eqref{eq:logcost} is determined by choosing $\mat{W}$ to be the matrix of the $m$ eigenvectors of the covariance matrix $\Sigma$ of ${\bf G}$ corresponding to the largest eigenvalues. 

In summary, whenever the data are sampled from a superstatisitcal distribution that is potentially heavy-tailed ($\alpha$-stable, multivariate-t, $\cdots$), one has to estimate the covariance matrix $\Sigma$ of the underlying Gaussian vector and not that of the data (the covariance matrix of the heavy-tailed data might not exist), and then extract the principal directions from the estimated matrix. Our results include multiple estimators of $\Sigma$ based on the observation of the data matrix $\mat{X}$. 

These results and the corresponding proofs are detailed in the next sections.  

\section{Main Results}
\label{sec:mr}
We start with a definition of the column space --also known as range or image, followed by a statement and proof of our first result in the form of Theorem~\ref{th:linopt}. 

\begin{definition}[Column space]
Given a matrix $\mat{W} \in\mathbb{R}^{d\times m}$, its \emph{column space} is defined by
\[
\mathrm{col}(\mat{W}) \;=\; \{\,\mat{W}{\bf v} \mid {\bf v}\in\mathbb{R}^m\}\subseteq\mathbb{R}^d.
\]
\end{definition}

\begin{theorem}
\label{th:linopt}
Let $\vect{X} \in \mathbb{R}^d$ be a random vector and let
\[
    F(\mat{M}) = \Ep{{\bf X}}{\ln \left( 1 + \snorm{ {\bf X} - \mat{W}\mat{M}^{\top}{\bf X} }\right)},
\]
where $\mat{W} \in\mathbb{R}^{d\times m}$ is a given matrix that satisfies $\mat{W}^\top \mat{W}=\mat{I}_m$.  Then
\[
\mat{W} = \underset{\mat{M} \in \mathbb{R}^{d\times m}}{\arg \min} \, F(\mat{M}).
\]
\end{theorem}

\begin{proof}
The proof proceeds in three steps.
\smallskip

\begin{itemize}
\item \textbf{Orthogonal projection onto $\mathrm{col}(\mat{W})$.}
Since $\mat{W}^\top \mat{W} = \mat{I}_m$, the matrix
\[
\mat{P} = \mat{W}\,\mat{W}^\top
\]
is the orthogonal projector onto the column space $\mathrm{col}(\mat{W}) \subset \mathbb{R}^d$.  Therefore, for any $\vect{x} \in \mathbb{R}^d$,
\[
\mat{P}\,{\bf x}
=\underset{{\bf v} \, \in \, \mathrm{col}(W)}{\arg \min}\, \left\lVert \vect{x} - {\bf v} \right\rVert_2,
\]
and the minimal distance is
\[
\left\lVert \vect{x} - \mat{P}\,\vect{x}\right\rVert_2
=\left\lVert \vect{x} - \mat{W}\,\mat{W}^\top \vect{x}\right\rVert_2,
\]
by the projection Theorem~\cite{luenberger1997optimization}.
\smallskip
\item \textbf{Error for general $\mat{M}$.}
Let $\mat{M}$ be any $d\times m$ matrix. 
Then
\[
\mat{W}\,\mat{M}^\top \vect{x}\;\in\;\mathrm{col}(\mat{W}),
\]
so by the projection theorem,
\[
\left\lVert \vect{x} - \mat{W}\,\mat{M}^\top \vect{x} \right\rVert_2
\ge
\left\lVert \vect{x} - \mat{W}\,\mat{W}^\top \vect{x}\right\rVert_2,
\]
for any $\vect{x} \in \mathbb{R}^d$ with equality iff $\mat{W}\,\mat{M}^\top \vect{x} = \mat{W}\,\mat{W}^\top \vect{x}$.
\smallskip
\item \textbf{Minimizing $F(\mat{M})$.}
Define the residual for any realization $\vect{x}$,
\[
{\bf r}(\mat{M})
= \vect{x} - \mat{W}\,\mat{M}^\top \vect{x}.
\]
Then the above implies
\[
\left\lVert {\bf r}(\mat{M})\right\rVert_2
\ge
\left\lVert{\bf r}(\mat{W}) \right\rVert_2.
\]

Since $\phi(u)=\ln(1+u^2)$ is strictly increasing in $u^2$, we have
\[
\phi\bigl(\left\lVert {\bf r}(\mat{M})\right\rVert_2\bigr)
\ge
\phi\bigl(\left\lVert {\bf r}(\mat{W})\right\rVert_2\bigr).
\]
By taking expectations on both sides, we get  
\[
F(\mat{M}) \geq F(\mat{W}),
\]
for any $\mat{M}$. Hence $F(\mat{W})$ is the global minimum of $F$, and $\mat{M} = \mat{W}$ is a minimizer.
\end{itemize}
\end{proof}

We next show the heavy-tailed PCA result. 

\begin{theorem}
\label{th:heavypca}

Let $\vect{X} = A^{\frac{1}{2}}\, \vect{G} \in \mathbb{R}^d$ be a superstatistical random vector where $A$ is positive, continuous RV with PDF $p_A(a)$ and $\vect{G} \sim \mathcal{N}(0,\Sigma)$ is a zero-mean Gaussian vector, with covariance matrix $\Sigma$. Let $m < d$. Then the ``best" $m$-dimensional features vector ${\bf V}$ under the logarithmic average loss~\eqref{eq:logcost} is  determined by the largest $m$ eigenvectors of the matrix $\Sigma$. 
\end{theorem}
\begin{proof}
Theorem~\ref{th:linopt} states that the best linear transformation of the feature vector ${\bf X}$ minimizing the logarithmic average loss~\eqref{eq:logcost} is the generator matrix $\mat{W}$ itself. Next, we find the optimal $\mat{W}$ and show that the best $m(<d)$-dimensional vector of features ${\bf V}$ is determined by projecting the $d$-dimensional vector of features ${\bf X}$ on the $m$ largest principal components of $\Sigma$.

We start by finding the optimal $\mat{W}$. Since ${\bf X} = A^{\frac{1}{2}} {\bf G}$, using iterated expectations, equation~\eqref{eq:logcost} boils down to 
\begin{equation}
    \Ep{A} { \Ep{\bf G} { \ln \left( 1 + a \snorm{ {\bf G} - \mat{W}\,\mat{W}^\top {\bf G} } \right) } }.
    \label{eq:argminlog}
\end{equation}
We proceed next, by fixing $a >0$ and solving the minimization problem for the inner expectation.  More specifically, we consider the following problem: let $a >0$, and find
\begin{equation}
   \underset{\mat{W}: \, \mat{W}^\top\mat{W} = \mat{I}_m}{\arg \min} \Ep{\bf G} { \ln \left( 1 + a \snorm{ {\bf G} - \mat{W}\,\mat{W}^\top {\bf G} } \right) }.
\label{eq:argminlogequiv}
\end{equation}

We show that the minimizer $\mat{W}$ for equation~\eqref{eq:argminlogequiv} is actually independent of $a$. Given the constraint  $\mat{W}^\top \mat{W} = \mat{I}_m$, there exists a unitary matrix \( \mat{U} \in \mathbb{R}^{d \times d} \) such that
\[
\mat{W} = \mat{U} 
\begin{bmatrix}
\mat{I}_m \\
0
\end{bmatrix},
\]
i.e., \( \mat{W} \) consists of the first \( m \) columns of \( \mat{U} \). The optimization in~\eqref{eq:argminlogequiv} is equivalent to one on $\mat{U}$ --the first $m$ columns of it that is. Note the eigenvalue decomposition of $\mat{W}\, \mat{W}^\top =  \mat{U} \mat{D} \mat{U}^\top$, where $\mat{D} \in \mathbb{R}^{d \times d}$ 
\[\mat{D} = 
\begin{bmatrix}
\mat{I}_m & 0 \\
0 & 0
\end{bmatrix},\] 
is a diagonal matrix that is {\em singular} since the rank of $\mat{W}$ is equal to $m < d$. 
Considering the expectation term in equation~\eqref{eq:argminlogequiv}, we write
\begin{align}
\Ep{\bf G} { \ln \left( 1 + a \snorm{ {\bf G} - \mat{W}\,\mat{W}^\top {\bf G} } \right) } & = 
\Ep{\bf G} { \ln \left( 1 + a \snorm{ {\bf G} - \mat{U}\mat{D}\mat{U}^\top {\bf G} } \right) } \nonumber \\
& = \Ep{\bf G} { \ln \left( 1 + a \snorm{ \mat{U}^\top {\bf G} - \mat{D}\mat{U}^\top {\bf G} } \right) } \label{eq:unitary}\\
& = \Ep{\bf G'} { \ln \left( 1 + a \snorm{ {\bf G}' - \mat{D} {\bf G}' } \right) } \label{eq:gprime}.
\end{align}
Equation~\eqref{eq:unitary} is justified since $\mat{U}^\top$ is a unitary matrix and in equation~\eqref{eq:gprime} we define ${\bf G}' = \mat{U}^\top {\bf G}$. The first $m$ diagonal elements of $\mat{D}$ are equal to $1$ and the last $(d - m)$ elements are equal to $0$. 

Equation~\eqref{eq:gprime} can be written as
\begin{equation*}
    \Ep{\bf G'}{ \ln \left( 1 + a \sum_{m+1}^d G^{'2}_i \right) },
\end{equation*}
and the minimization problem~\eqref{eq:argminlogequiv} becomes 
\begin{equation}
\label{eq:argmin}
    \underset{\mat{U}}{\arg \min} \,\, \Ep{\bf G'}{ \ln \left( 1 + a \sum_{m+1}^d G^{'2}_i\right) },
\end{equation}
over all unitary matrices and where ${\bf G'} = \mat{U}^\top {\bf G}$. For  $(m + 1) \leq i \leq d$ define $G'_i=\sigma_i G''_i$ where $G''_i \sim\mathcal{N}(0,1)$ and $\sigma_i$ represents the standard deviation of $i$-th diagonal element of $\Sigma'$, the covariance matrix of ${\bf G'}$. 
Rewriting equation~\eqref{eq:argmin}, we get the following equivalent problem:
\begin{align*}
    & \underset{{\sigma_i, m+1 \leq i \leq d}}{\arg \min}\,\, f(\sigma_{m+1}, \cdots, \sigma_d) \\
    \text{where } \qquad & f(\sigma_{m+1}, \cdots, \sigma_d) \triangleq \Ep{\bf G^{''}} {\ln \left( 1 + a \sum_{m+1}^d \sigma_i^2 {G''}_i^2 \right) }.
\end{align*}

Taking the partial derivative with respect to some $\sigma_{\ell}$, $(m+1) \leq \ell \leq d$, we obtain
\begin{equation}
    \label{eq:partialder}
    \frac{\partial f}{\partial \sigma_\ell} = \Ep{{\bf G}^{''}} {\frac{2\, a\, \sigma_\ell{G''}_\ell^2}{1  + a\sum_{m+1}^d{\sigma_i^2G''}_i^2}} \geq 0,
\end{equation}
which are non-negative for all $(m+1) \leq i \leq d$. Therefore, $f(\sigma_{m+1}, \cdots, \sigma_d)$ is increasing with respect to each $\sigma_i$, 
 and this in turn, implies that $\mat{U}$ has to be chosen in such a way that the last $(d - m)$ marginal variables $G^{'}_i$, $(m+1) \leq i \leq d$  of ${\bf G}^{'} = \mat{U}^{\top} {\bf G}$ have the lowest variances among all the $d$ variables of ${\bf G}^{'}$. Among all orthogonal transforms $\mat{U}$, this is achieved when $\mat{U}$ is chosen as the eigenvector matrix of the covariance matrix $\Sigma$ of ${\bf G}$, with eigenvalues ordered decreasingly (Fan's maximal principle~\cite[Thm. 3.13]{meckes2019matrix}). 
 
 Since 
\(
\mat{W} = \mat{U} 
\begin{bmatrix}
\mat{I}_m &
0
\end{bmatrix}^\top
\), this implies that 
\[{\bf V} = \mat{W}^\top {\bf X} = \begin{bmatrix}
\mat{I}_m &
0
\end{bmatrix} \mat{U}^\top {\bf X} = A^{1/2} \begin{bmatrix}
\mat{I}_m &
0
\end{bmatrix} \mat{U}^\top {\bf G}.\] 

 The retained directions of the $m$-dimensional subspace are the eigenvectors corresponding to the highest eigenvalues of $\Sigma$, thus leaving out the $(d - m)$ eigenvectors corresponding to the lowest eigenvalues. 
\end{proof}

\section{Methods}
\label{sec:method}
\subsection{Estimating the Covariance Matrix}
\label{sec:estimate}
 According to the results of the previous section, performing a reduction in dimensionality for sub-Gaussian data under the logarithmic loss defined in~\eqref{eq:logcost} is optimally done by finding the eigenvectors of the covariance matrix of the underlying Gaussian vector, and hence it is necessary to estimate this matrix from the heavy-tailed data. This is directly related to finding the directions for the elliptical data.

Next, we propose two different methods to estimate the {\em shape} matrix of the sub-Gaussian data or, equivalently, the covariance matrix of the underlying Gaussian vector $\Sigma$. For the remainder of this section, we assume, WLoG, that the data are zero-centered, and we present the methods that can be used to estimate $\rho_{ij}$, the correlation of $G_i$ and $G_j$.  

\subsubsection{First method: Ratio of the marginals}
Starting with a sub-Gaussian vector ${\bf X} = A^{\frac{1}{2}} {\bf G}$, we form $\frac{d(d-1)}{2}$ ratios by dividing the row variables $X_i/X_j$, pairwise, $1 \leq i <  j \leq d$, which will eliminate $A^{\frac{1}{2}}$ since it is a common factor of both terms. The resultant variable is a ratio of two Gaussian variables, which is well-known to follow a Cauchy distribution $\mathcal{C}(\mu_{ij},\gamma_{ij})$~\cite{curtiss1941distribution} where 
\begin{equation}
  \mu_{ij} = \rho_{ij} \frac{\sigma_i}{\sigma_j}, \qquad \gamma_{ij} = \frac{\sigma_i}{\sigma_j}\sqrt{1-\rho_{ij}^2}, \label{eq:divcau}
\end{equation}
where $\sigma_i$, $\sigma_j$, and $\rho_{ij}$ represent the standard deviations and the correlation coefficient between the numerator and denominator- the $i$-th and $j$-th entries of ${\bf G}$ respectively. 

We start by estimating the parameters of the $\frac{d(d-1)}{2}$ Cauchy variables $\mathcal{C}(\mu_{ij},\gamma_{ij})$, $1 \leq i <  j \leq d$. This can be done using the {\em power estimators} and {\em location estimators} proposed in~\cite{sayde2025heavy}. Once that $\left\{\mu_{ij}\right\}$'s and $\left\{\gamma_{ij}\right\}$'s are computed, in order to find the $\left\{ \rho \right\}_{ij}$'s, it remains to compute the ratio $\sigma_i/\sigma_j$. Multiple methods can be applied to this end. 
Once more, one can use, for example, the power estimator of~\cite{sayde2025heavy} and calculate the {\em power} of the marginal random variables $X_i$ and $X_j$, denoted by $\gamma_i$ and $\gamma_j$ from the marginal observed heavy-tailed data. Since the power or ``scale" of $X_i = A^{\frac{1}{2}} G_i$ is proportional to $\sigma_i$, an estimate of the ratio $\frac{\sigma_i}{\sigma_j}$ is found using the fact that $\frac{\sigma_i}{\sigma_j} = \frac{\gamma_i}{\gamma_j}$. More details about the power and location definitions are found in~\cite{sayde2025heavy,FAF2018}. 

After estimating $\mu_{ij}$, $\gamma_{ij}$, and $\sigma_i/\sigma_j$,  and when it comes to evaluating the $\{\rho_{ij}\}$'s, we can proceed using one of the following 3 different formulas that follow directly from~\eqref{eq:divcau}:

\begin{eqnarray}
    \rho_{ij} &=&\frac{\mu_{ij}}{\sqrt{\mu_{ij}^2+\gamma_{ij}^2}} \label{eq:method1a}\\
    \rho_{ij} &=& \text{sgn}(\mu_{ij})\sqrt{1-\frac{\sigma_j^2}{\sigma_i^2}\gamma_{ij}^2} \label{eq:method1b}\\
    \rho_{ij} &=& \mu_{ij} \frac{\sigma_j}{\sigma_i} 
\label{eq:method1c}
\end{eqnarray}

In Section~\ref{sec:numres}, we numerically compare the accuracy of the three previous formulas to compute the $\{\rho_{ij}\}$s. 

\subsubsection{Second method: The log-correlation}

We define the log-correlation between two random variables as being the average of the product of their logarithm. More specifically, let $X$, $Y$ be two random variables, their log-correlation is given by: 
\begin{equation}
\label{eq:logcorr}
\ell_{XY} = \text{E}\left[\log{|X|}\log{|Y|}\right].
\end{equation}

The log correlation is finite for a large class of distributions where the standard correlation does not exist, such as the class of power laws with characteristic exponent $\alpha < 2$.  

This second method is based on the observation that there is a one-to-one map (up to a sign) between the log-correlation $\ell_{G_{i}G_{j}}$ and the correlation coefficient $\rho_{ij}$ between $G_i$ and $G_j$, $1 \leq i < j \leq d$, elements of the $d$-dimensional Gaussian vector ${\bf G} \sim\mathcal{N}(0,\Sigma)$. As such, we evaluate $\ell_{G_{i}G_{j}}$ by estimating the log-correlation $\ell_{X_{i}X_{j}}$ of the observed data, and use a precomputed lookup table that maps the value $\ell_{G_{i}G_{j}}$ to the corresponding $\left|\rho_{ij}\right|$. In what follows, we assume WLoG that all $G_i$s have unit variance. In the general case, one can estimate the variances of the marginals using the power estimator and normalize the entries accordingly. 

To compute the empirical log-correlation coefficient, we first take the log of the data vector in absolute value, then we multiply the outcomes pairwise. We get 
\begin{align}
\log{|X_i|} \, \log{|X_j|}
& = \log \left( A^\frac{1}{2}|G_i| \right) \log \left( A^\frac{1}{2}|G_j| \right)\nonumber\\
& = \left[ \frac{1}{2}\log{A}+\log{|G_i|} \right] \left[ \frac{1}{2}\log{A}+\log{|G_j|} \right] \nonumber\\
& = \frac{1}{4} \left[ \log{A} \right]^2 + \frac{1}{2} \big[ \log{|G_i|} + \log{|G_j|} \big] \log{A} + \log{|G_i|} \, \log{|G_j|}. \label{eq:logcorr1}
\end{align}
Taking the expected value in equation~\eqref{eq:logcorr1}, we get for all $1 \leq i < j \leq d$,
\begin{align*}
    \E{ \log{|X_i|} \, \log{|X_j|}} = \frac{1}{4} \E{ \left[ \log{A} \right]^2} + \E{ \log{A} } \E{ \log{|G_i|} } + \E{ \log{|G_i|} \, \log{|G_j|}},
\end{align*}
where we used the fact that $A$ and ${\bf G}$ are independent and that $\E{\log{|G_i|})} = \E{\log{|G_j|}}$ since $G_i \sim G_j$ by assumption.
Rearranging the terms, we obtain
\begin{equation}
    \E{\log{|G_i|} \, \log{|G_j|}} = \E{ \log{|X_i|} \, \log{|X_j|} } - \frac{1}{4} \E{\left[ \log A \right]^2 } - \E{ \log{A} } \E{ \log{|G_i|} }. \label{eq:lookup}
\end{equation}

The RHS of equation~\eqref{eq:lookup} is computed as follows: $\E{\log{|X_i|}\log{|X_j|}}$ can be calculated as an empirical mean of the pairwise product of the log of the marginal data. For a given $A$ and for unit variance $\{G_i\}$'s, the remaining two terms are independent of the data and of $\rho_{ij}$. Finally, a lookup table is computed that gives the value of $|\rho_{ij}|$ for each value of $\E{\log{|G_i|}\log{|G_j|}}$. 

We have computationally built such a table by evaluating the logarithmic correlation $\E{\log{|G_1|}\log{|G_2|}}$ between two correlated Gaussian variables with zero mean and unit variance. A step size of $10^{-5}$ was used for values of $\rho \in [0,1]$.

\subsection{Numerical Example}
\label{sec:numres}

In this section, we validate and compare the performance of the two estimation methods proposed in Section~\ref{sec:estimate}. We consider a $2$-dimensional setup where we generate data according to an elliptical Cauchy distribution $\mathbf{\mathcal{C}} = A^{\frac{1}{2}} {\bf G}$, where $A \sim \mathcal{S} \left(
\frac{1}{2}, 1, \frac{1}{2}, 0 \right)$ is a L\'evy distribution --alpha-stable with $\alpha = \frac{1}{2}$, and ${\bf G} \sim \mathcal{N}(0,\Sigma)$. As an example, we consider the following covariance matrix
\begin{equation}
    \Sigma = \begin{pmatrix}
                16 & 8 \rho \, \\ 8 \rho \,  & 4
    \end{pmatrix}
    \label{eq:estrho}
\end{equation}
and we sample $n = 800$ Cauchy vectors. We repeat the experiment $N_{\text{runs}} = 200$ times and apply the first and second methods to get an average estimate of $\Sigma$. The experiment is repeated for values of $\rho$ ranging from $0.1$ to $0.9$ where we evaluated the percentage relative error in the estimates of $\rho$. 

\begin{figure}[!htp]
    \centering

    \subfloat[First method: all equations.]{%
        \includegraphics[width=0.48\linewidth]{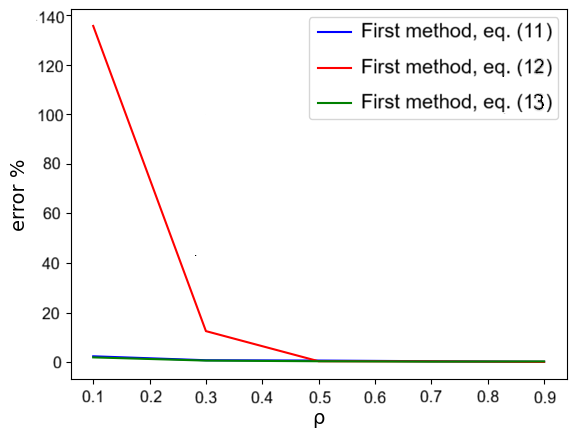}}
    \hfill
    \subfloat[Figure (a) for $\rho \geq 0.5$.]{%
        \includegraphics[width=0.48\linewidth]{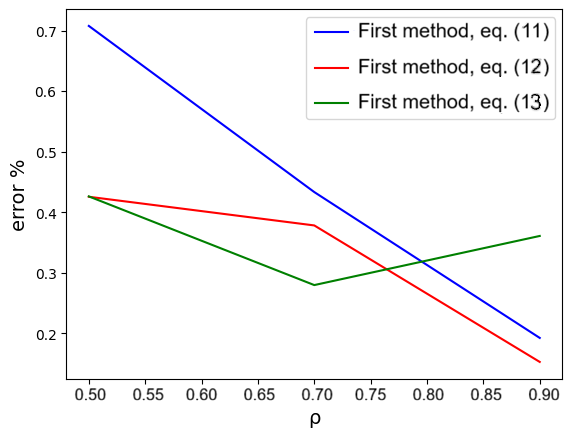}}

    \vspace{1mm}

    \subfloat[First method: equations~\eqref{eq:method1a} and~\eqref{eq:method1c}.]{%
        \includegraphics[width=0.48\linewidth]{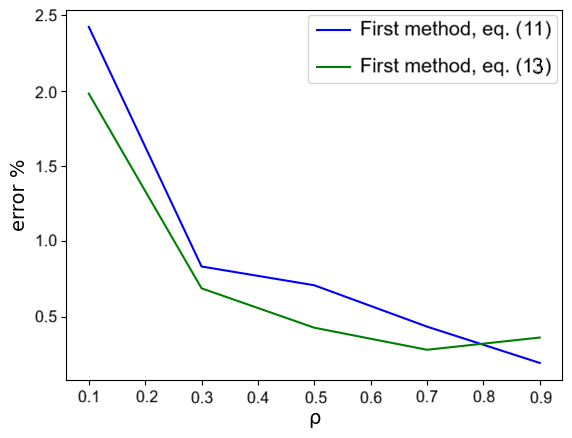}}
    \hfill
    \subfloat[First vs. second methods.]{%
        \includegraphics[width=0.48\linewidth]{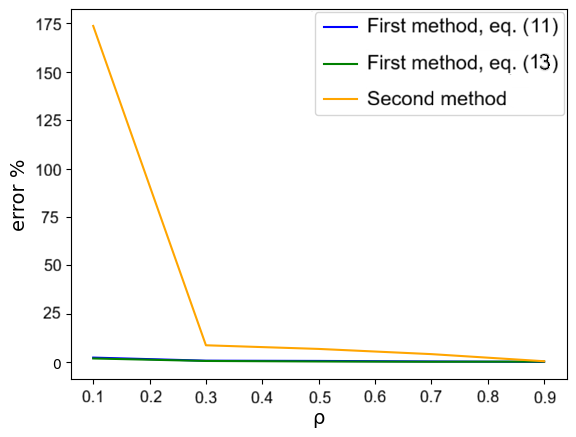}}

    \vspace{1mm}

    \subfloat[Figure (d) for $\rho \geq 0.3$.]{%
        \includegraphics[width=0.48\linewidth]{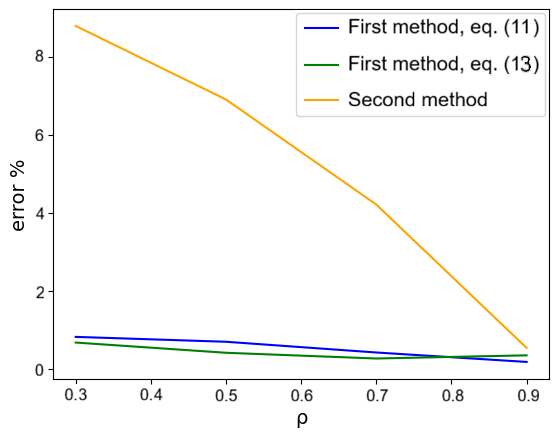}}

    \caption{Estimating $\rho$ in~\eqref{eq:estrho} using first and second methods.}
    \label{fig:mthd123}
\end{figure}

Figure~\ref{fig:mthd123}(a) clearly shows that using equation~\eqref{eq:method1b}  is inaccurate when estimating $\rho$ for values lower than $0.5$; this may be due to the inaccuracies in estimating the scale $\gamma$. On the other hand, Figure~\ref{fig:mthd123}(c)  shows that using the first method with either equation~\eqref{eq:method1a} or~\eqref{eq:method1c} is accurate on the full range with a maximum relative error of $2.4\%$ and $2\%$, respectively. We also note that the accuracy of the estimation obtained using equation~\eqref{eq:method1c} of the first method slightly outperforms equation~\eqref{eq:method1a} for values of $\rho$ below $0.8$ as shown in Figure~\ref{fig:mthd123}(c). The opposite is true for $\rho > 0.8$.

We also compare the log-correlation method (second method) with the first one. As shown in Figure~\ref{fig:mthd123}(d), the performance of the second method severely deteriorates for small values of $\rho$ (less than $0.3$). For values of $\rho \geq 0.5$, the error is in the range $[1\%;7\%]$.
According to these results, we recommend using equation~\eqref{eq:method1a} or equation~\eqref{eq:method1c} of the first method as they are the most accurate over a wide range of values of $\rho$.

We next compare our method --using equation~\eqref{eq:method1c} of the first method-- against two widely-used covariance estimators:
the empirical covariance (classical PCA) and Tyler’s scatter estimator. 
For consistency, we keep the same experimental setup as above and evaluate 
the various methods over the full range $\rho \in [0.1,0.9]$. 

When using Cauchy data (shown in Figure~\ref{fig:heavy_gaussian_compare}(a)), the empirical covariance was the most unstable where 
its error decreases from more than $30\%$ when $\rho=0.1$ to around $2\%$ at $\rho=0.9$. Tyler’s estimator is more stable but still shows errors between $3\%$ and $7\%$ for moderate correlations. In contrast, our method maintains an error below approximately $3\%$ for all values of $\rho$, and below $1\%$ once $\rho$ increases $\rho \ge 0.3$.

When the data are less heavy-tailed ($\alpha = 1.7$), the 
performance gap narrows as we can see in Figure~\ref{fig:heavy_gaussian_compare}(b) and Tyler's method starts showing robustness.  
The empirical covariance still exhibits high errors at small values of $\rho$ that is, around $20\%$ at $\rho = 0.1$. Our method remains the most accurate across the entire range, with errors below $2\%$ and already under $1\%$ for $\rho \ge 0.3$.

In the Gaussian setting, all three methods perform similarly well as shown in Figure~\ref{fig:heavy_gaussian_compare}(c), the error is roughly below $3.5\%$ even at $\rho = 0.1$, and drops under $0.5\%$ for 
$\rho \ge 0.5$. Here, the empirical covariance is slightly more accurate, as expected.

Our method is seen to preserve performance under Gaussian data while offering a clear advantage in heavy-tailed regimes ($\alpha$ $<$ 2), thus providing a robust and broadly applicable alternative to the classical covariance estimator and Tyler’s estimator. 

\begin{figure}[!t]
    \centering

    \subfloat[Cauchy data.]{%
        \includegraphics[width=0.48\linewidth]{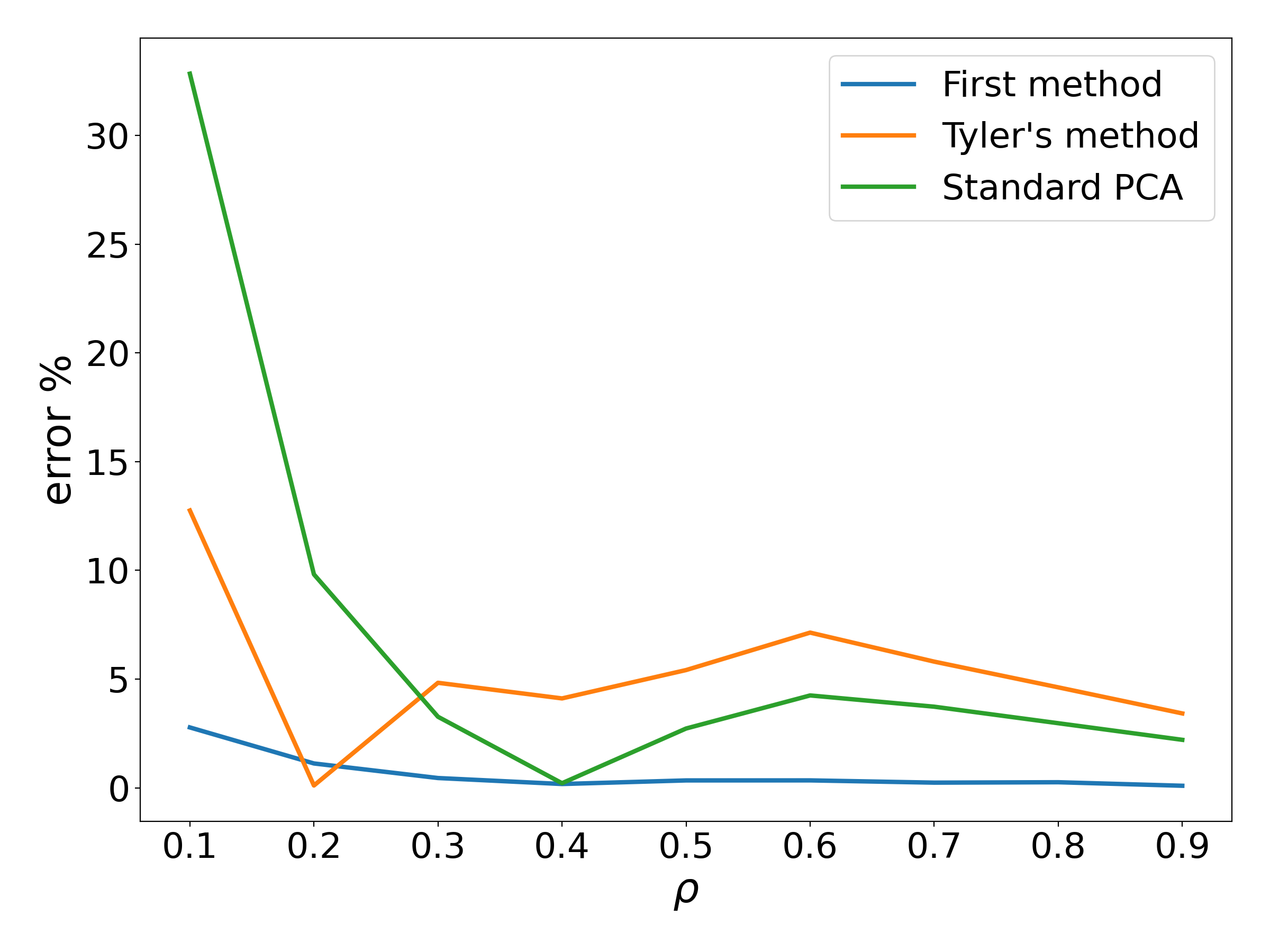}}
    \hfill
    \subfloat[$\alpha$-stable data with $\alpha = 1.7$.]{%
        \includegraphics[width=0.48\linewidth]{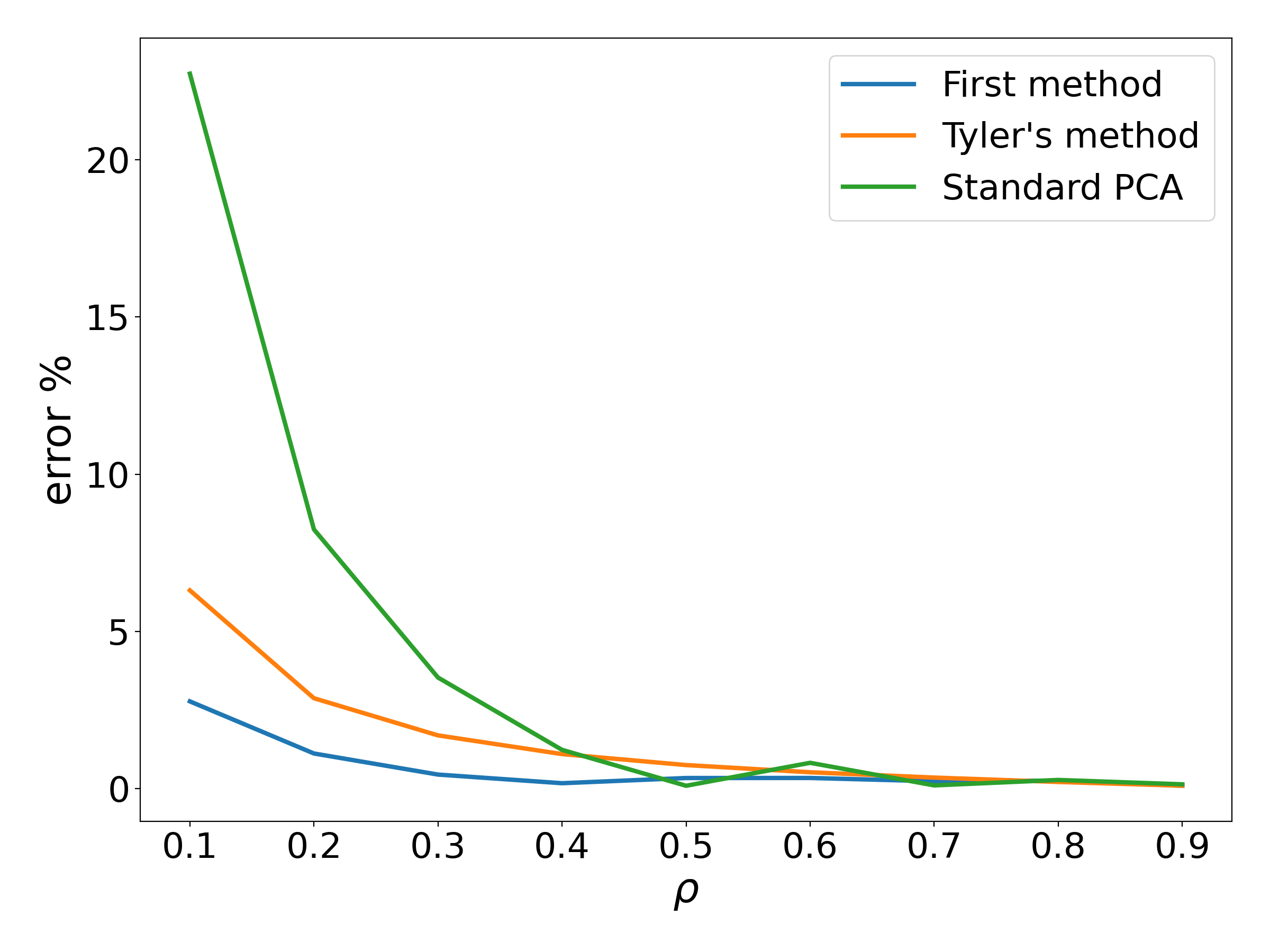}}

    \vspace{1mm}

    \subfloat[Gaussian data.]{%
        \includegraphics[width=0.48\linewidth]{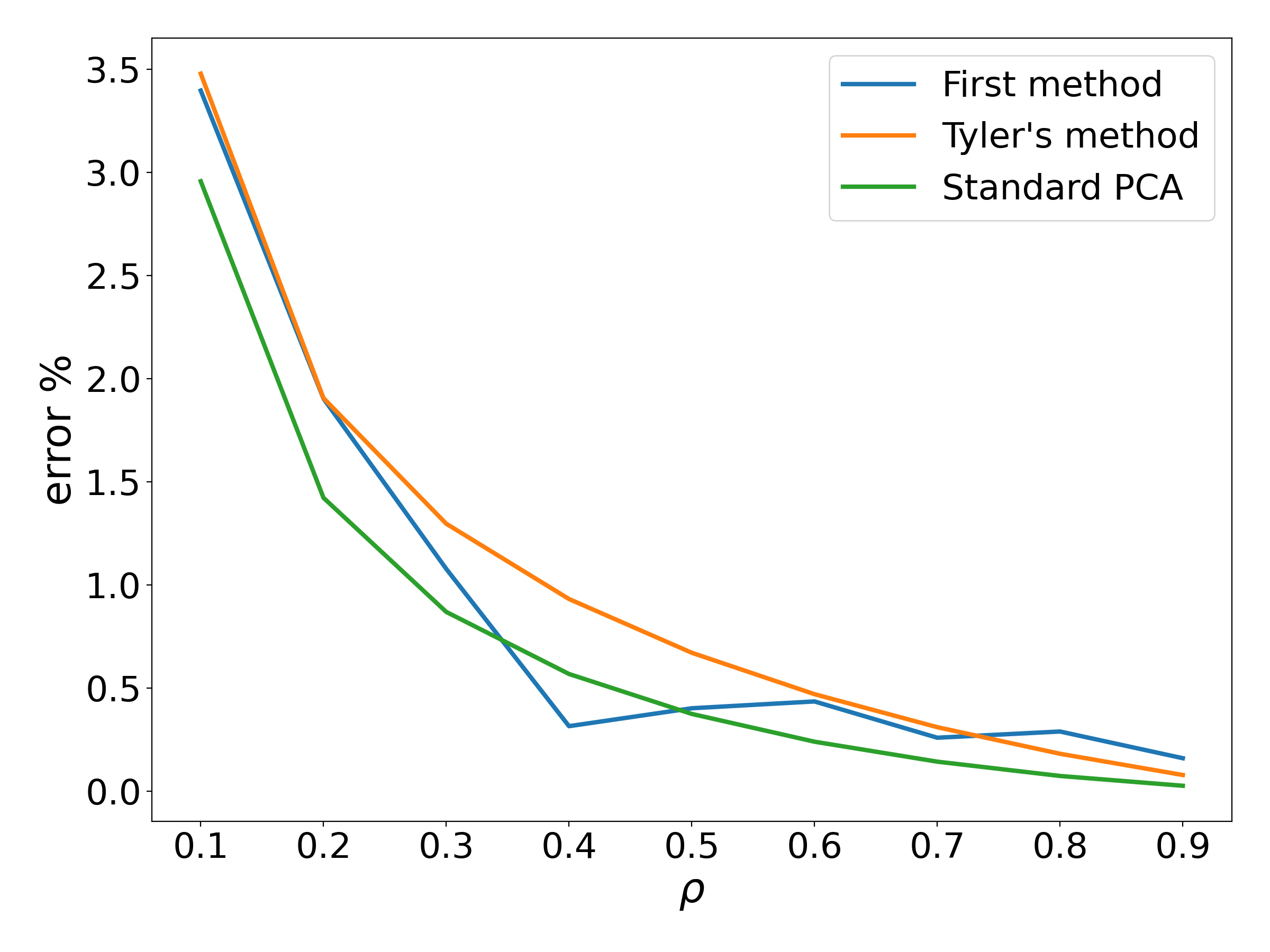}}

    \caption{Comparison between our method,
    Tyler and the PCA's covariance estimator under heavy-tailed and Gaussian data.}
    \label{fig:heavy_gaussian_compare}
\end{figure}

\subsubsection*{Another Example}

In another example, we generate two-dimensional data sampled from a sub-Gaussian vector with $\alpha = 0.7$, and $\Sigma = RDR$ where $R=\begin{pmatrix}
    1&0.8\\0.8&1
\end{pmatrix}$
and $D= \begin{pmatrix}
    1&0\\
    0&0.4
\end{pmatrix}$. In Figure~\ref{fig:PCvisualization}, we depict the first principal component (PC1) obtained using equation~\eqref{eq:method1c} in addition to the one recovered by the classical PCA for the sub-Gaussian data. The cosine similarity between the true PC1 and that estimated by equation~\eqref{eq:method1c} is approximately 0.9993, whereas it is around $0.7149$ for the classical PCA and the true PC1. This discrepancy highlights the sensitivity of the sample covariance to heavy-tailed data, whereas the proposed estimator reliably recovers the dominant eigen direction, especially under sub-Gaussian regimes.

\subsubsection*{More than two dimensions}

Finally, we test our method on $n = 1,000$ data samples generated according to a 3-dimensional sub-Gaussian Cauchy vector with a covariance matrix $\Sigma$ of the underlying Gaussian vector given by:

\[\Sigma = 
\begin{pmatrix}
1 &0.9&0.5 \\
0.9&1 & 0.2\\
0.5&	0.2&	1
\end{pmatrix}.\]
 The experiment is repeated $N_{\text{runs}} = 400$ times. The resulting bias and RMSE for each entry in the matrix ($\hat{\Sigma} - \Sigma$ ) are found to be in the range $[0.005;0.0095]$ and $[0.05;0.12]$ respectively. 
 
 \[\text{Bias} =\begin{pmatrix}
0.0048 & 0.0077 & 0.0056 \\
0.0077 & 0.0095 & 0.0052 \\
0.0056 & 0.0052 & 0.0061
\end{pmatrix} \qquad 
\text{RMSE} = 
\begin{pmatrix}
0.12 & 0.11 & 0.077 \\
0.11 & 0.12 & 0.052 \\
0.077 & 0.052 & 0.12
\end{pmatrix}
\]

    \begin{figure}[!t]
    \centering
    \includegraphics[width=0.70\textwidth]{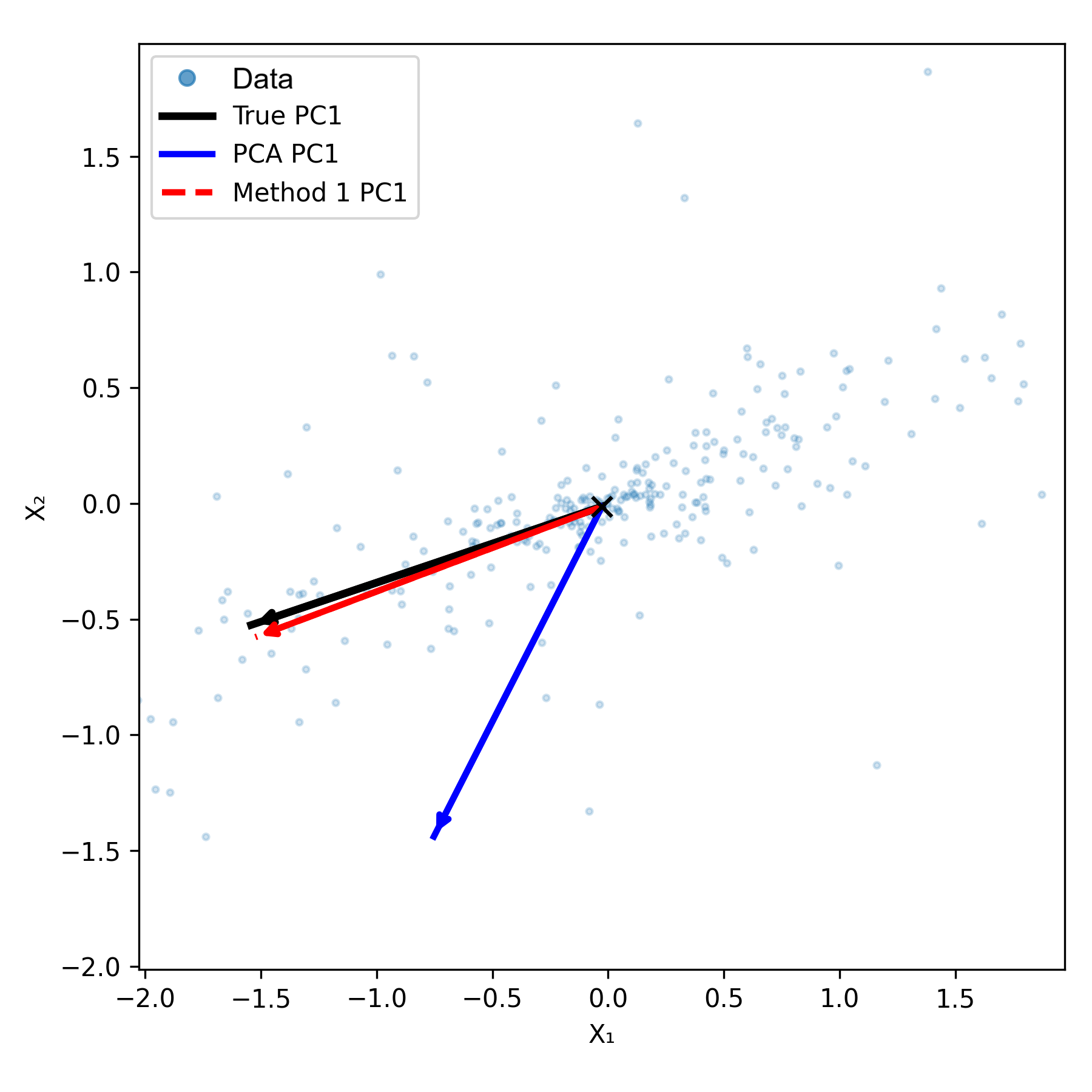}
    \caption{Visualization of the first principal component (PC1) obtained using equation~\eqref{eq:method1c} and standard PCA, compared with the true PC1 of the sub-Gaussian data.}
    \label{fig:PCvisualization}
\end{figure}

\section{Experiments and Results: Background denoising}
\label{sec:ExpRes}
\subsection{MNIST images}
40 MNIST images of the digits ``0" and  ``8" --with $784$ pixels each-- were corrupted with additive heavy-tailed noise. The 40 selected images were chosen in such a way as to highly resemble each other. First, we convert the 40 images into a single matrix $(784 \times 40)$ denoted by $\mat{X}$. Then, we added $40$ independent heavy-tailed ``superstatistical" correlated noises to each of the $40$ figures. More specifically, we generated $40$ independent sub-Gaussian random vectors of the form $A^{\frac{1}{2}} {\bf G}$, each with dimension $d = 784$, which are configured into a noise matrix $\mat{N}$ whose dimensions are $(784 \times 40)$, and we added $\mat{N}$ to $\mat{X}$. 

Instead of computing the mean (which might not be well defined for heavy-tailed data), we evaluated the location parameter proposed in~\cite{FAF2018,sayde2025heavy} for each pixel across the 40 pictures, then we constructed an image of ``locations" representing an ``average" image, which we subtracted from each of the 40 images.
We then reconstruct each image by rank-$k$ projection onto components learned from the noisy set using either the classical PCA or the proposed heavy-tailed PCA. More specifically, for the heavy-tailed PCA, we estimate the covariance matrix $\Sigma$ of the noisy images using the first method presented in Section~\ref{sec:method} and using equation~\eqref{eq:method1a}. In all the results shown subsequently, we only retain the largest $k = 10$ principal components.

We conducted this experiment for two different models of heavy-tailed noise,
\begin{itemize}
    \item a sub-Gaussian Cauchy and 
    \item a Student-t with DoF $\nu = 1.2$
\end{itemize}
vectors, both with an underlying IID Gaussian vector with covariance matrix $\Sigma = 10 \,\, \mat{I}_{784}$\footnote{where $\mat{I}_m$ denotes the $m\times m$ identity matrix}. The simulation results are shown in Figures~\ref{fig:0s} and~\ref{fig:8st}. Both figures show that the proposed heavy-tailed approach achieves a higher-fidelity denoising, indicating superior suppression of large impulsive deviations. Visually, heavy-tailed PCA removes ``salt-and-pepper" artifacts and yields cleaner backgrounds and sharper strokes as shown in Figures~\ref{fig:0hr_s} and \ref{fig:8hr_st}, while classical PCA leaves residual speckles that intensify as the tails grow heavier, as depicted in Figures~\ref{fig:0sr_s} and~\ref{fig:8sr_st}. 

Finally, the heavy-tailed PCA is not specific to potential heavy-tailed noise, as it performs well even under ``thin-tailed" noise. In Figure~\ref{fig:8g}, we insert an IID Gaussian noise with $\Sigma = 10 \,\, \mat{I}_{784}$ into the original image and apply a standard and heavy-tailed PCA reconstruction. The results are visually similar, as can be seen in Figures~\ref{fig:8sr_g} and~\ref{fig:8hr_g}.

\begin{figure*}[!t]
    \centering

    \subfloat[Original image.]{%
        \includegraphics[width=0.23\textwidth,keepaspectratio]{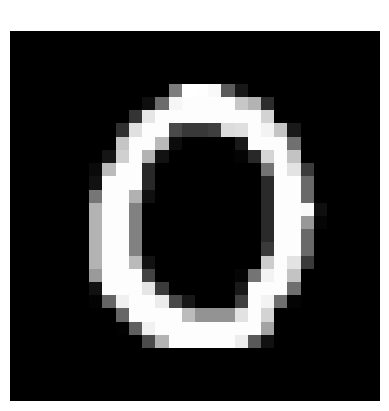}
        \label{fig:0og_s}
    }
    \hfill
    \subfloat[Noisy image with Cauchy noise.]{%
        \includegraphics[width=0.23\textwidth]{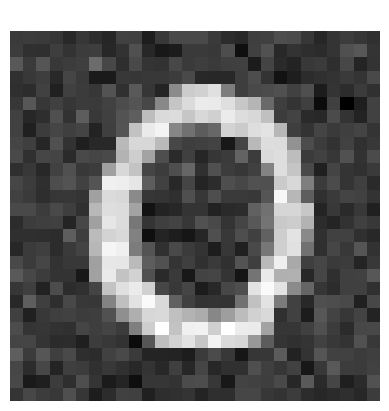}
        \label{fig:0n_s}
    }
    \hfill
    \subfloat[Standard PCA reconstruction.]{%
        \includegraphics[width=0.23\textwidth]{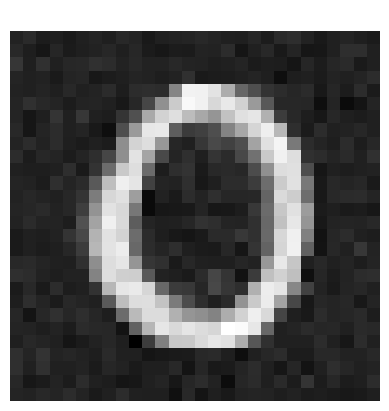}
        \label{fig:0sr_s}
    }
    \hfill
    \subfloat[Heavy-tailed PCA reconstruction.]{%
        \includegraphics[width=0.23\textwidth]{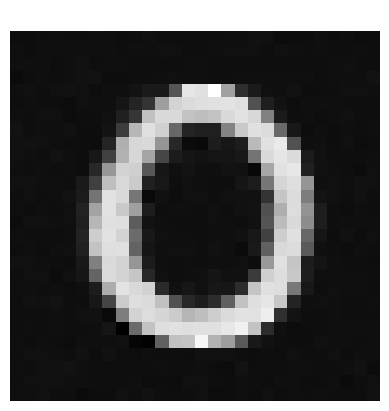}
        \label{fig:0hr_s}
    }

    \caption{Image reconstruction under Cauchy noise. From left to right: original image, noisy image, standard PCA reconstruction, and heavy-tailed PCA reconstruction.}
    \label{fig:0s}
\end{figure*}

\begin{figure*}[!t]
\centering

\subfloat[Original image.]{%
    \includegraphics[width=.23\textwidth,keepaspectratio]{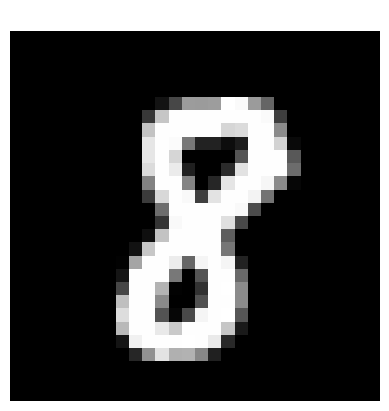}
    \label{fig:8og_st}
}
\hfill
\subfloat[Noisy image with Student-t noise.]{%
    \includegraphics[width=.23\textwidth]{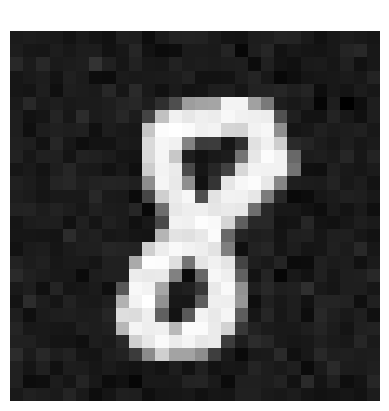}
    \label{fig:8n_st}
}
\hfill
\subfloat[Standard PCA reconstruction.]{%
    \includegraphics[width=.23\textwidth]{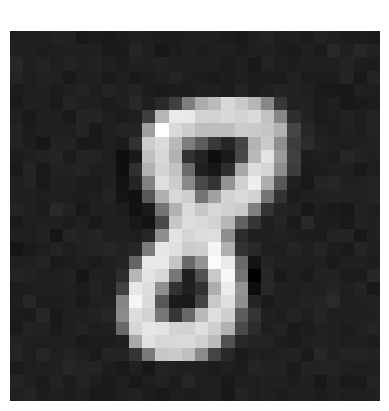}
    \label{fig:8sr_st}
}
\hfill
\subfloat[Heavy-tailed PCA reconstruction.]{%
    \includegraphics[width=.23\textwidth]{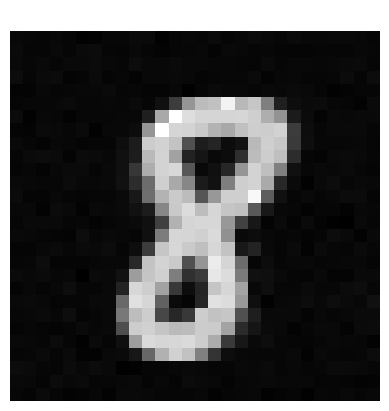}
    \label{fig:8hr_st}
}

\caption{Image reconstruction under Student-\(t\) noise with \(\nu=1.2\) degrees of freedom. From left to right: original image, noisy image, standard PCA reconstruction, and heavy-tailed PCA reconstruction.}
\label{fig:8st}
\end{figure*}

\begin{figure*}[!t]
\centering

\subfloat[Original image.]{%
    \includegraphics[width=.23\textwidth,keepaspectratio]{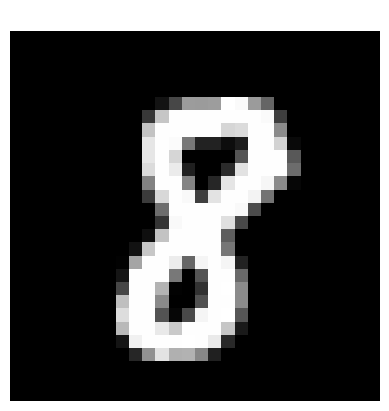}
    \label{fig:8og_g}
}
\hfill
\subfloat[Noisy image with Gaussian noise.]{%
    \includegraphics[width=.23\textwidth]{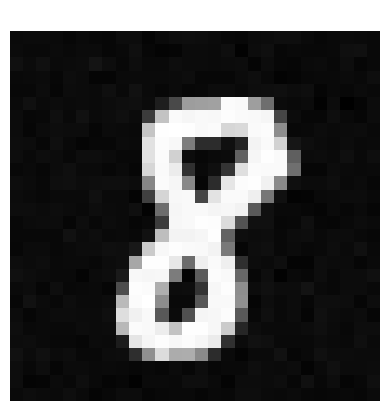}
    \label{fig:8n_g}
}
\hfill
\subfloat[Standard PCA reconstruction.]{%
    \includegraphics[width=.23\textwidth]{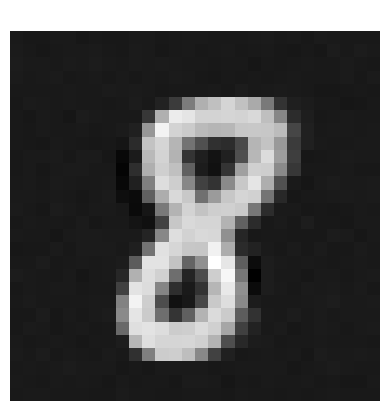}
    \label{fig:8sr_g}
}
\hfill
\subfloat[Heavy-tailed PCA reconstruction.]{%
    \includegraphics[width=.23\textwidth]{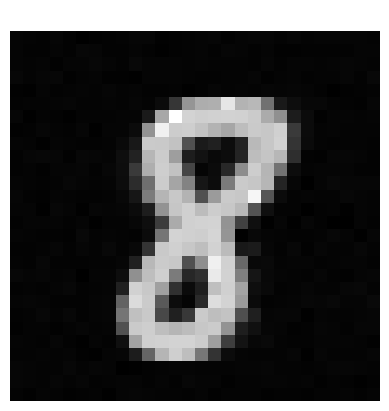}
    \label{fig:8hr_g}
}

\caption{Image reconstruction under Gaussian noise. From left to right: original image, noisy image, standard PCA reconstruction, and heavy-tailed PCA reconstruction.}
\label{fig:8g}
\end{figure*}

\subsection{Video data}
We carried out a controlled experiment on actual video data to assess the robustness of heavy-tailed PCA for background extraction and denoising. To create a small but difficult setting where noise and compression artifacts are more noticeable, a brief video was taken from YouTube, converted to grayscale, and then down-sampled to a low spatial resolution of 64 × 48. First, we convert 150 frames (corresponding to a 5 seconds duration with a frame-rate of 30 frames per second) into a $\left(3072 \times 150\right)$ matrix denoted $\mat{X}$.
Then, we generated 150 independent sub-Gaussian (Cauchy) random vectors, each with dimension d = 3072 and whose underlying Gaussian distribution has a covariance matrix $\Sigma =  0.02\times \mat{I}_{3072}$ that are configured to a noise matrix $\mat{N}$ whose dimensions are $\left( 3072 \times 150\right)$. Finally, we add $\mat{N}$ to $\mat{X}$.

Similarly to the MNIST application, we did not compute the mean of the frames, but instead constructed a frame of locations~\cite{sayde2025heavy,FAF2018} representing an average frame, which we subtracted for all frames. We then reconstruct each image by rank-$k$ projection onto components learned from the noisy set using either the classical PCA or the proposed heavy-tailed PCA. More specifically, for the heavy-tailed PCA, we estimate the covariance matrix $\Sigma$ of the noisy images using equation~\eqref{eq:method1a} of the first method presented in Section~\ref{sec:method}. We only retain the  principal component corresponding to the largest eigenvalue ($k = 1$).
The simulation results are shown in Figure \ref{fig:frame1pca} for two different frames. Both results show that the proposed heavy-tailed PCA removes noise artifacts and provides a ``cleaner" background extraction. 

Overall, these results demonstrate that the heavy-tailed PCA provides a more robust representation of the underlying video structure in the presence of significant noise and compression artifacts. In particular, its superior performance at low ranks highlights its effectiveness for background extraction and denoising in challenging, low-resolution video settings where classical PCA struggles.

\section{Conclusions}
\label{sec:conc}
This work provides a principled and general treatment of PCA for high-dimensional data with potentially infinite variance. By adopting a logarithmic loss for the superstatistical Gaussian subordination framework, we show that robust dimensionality reduction can be achieved without resorting to ad hoc trimming, sparsity assumptions, or finite-moment conditions. Our theoretical analysis establishes that the principal components of heavy-tailed data are governed by the covariance structure of the latent Gaussian generator, thereby justifying the use of standard PCA once this covariance is consistently estimated. We propose practical estimators for this purpose and demonstrate, through simulations and background denoising experiments, their superior robustness and accuracy compared to both empirical covariance-based PCA and Tyler-type methods. Beyond PCA, the proposed framework opens new avenues for robust subspace learning, detection, and representation in modern data regimes characterized by heavy tails, dependence, and impulsive perturbations. Extensions to online settings, tensor-valued data, and connections with information-theoretic objectives such as entropy and divergence-based learning under heavy-tailed distributions represent some directions for future work.
\begin{figure*}[!t]
\centering

\subfloat[Original frame.]{%
    \includegraphics[width=.235\textwidth,keepaspectratio]
    {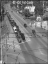}
    \label{fig:ogf1}
}
\hfill
\subfloat[Noisy frame with Cauchy noise.]{%
    \includegraphics[width=.235\textwidth]
    {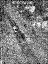}
    \label{fig:nf1}
}
\hfill
\subfloat[Standard PCA reconstruction $k=1$.]{%
    \includegraphics[width=.235\textwidth]
    {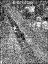}
    \label{fig:sp11}
}
\hfill
\subfloat[Heavy-tailed PCA reconstruction $k=1$.]{%
    \includegraphics[width=.235\textwidth]
    {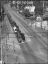}
    \label{fig:hp11}
}

\vspace{1mm}

\subfloat[Original frame.]{%
    \includegraphics[width=.235\textwidth,keepaspectratio]
    {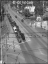}
    \label{fig:ogf8}
}
\hfill
\subfloat[Noisy frame with Cauchy noise.]{%
    \includegraphics[width=.235\textwidth]
    {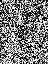}
    \label{fig:nf8}
}
\hfill
\subfloat[Standard PCA reconstruction $k=1$.]{%
    \includegraphics[width=.235\textwidth]
    {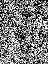}
    \label{fig:sp18}
}
\hfill
\subfloat[Heavy-tailed PCA reconstruction $k=1$.]{%
    \includegraphics[width=.235\textwidth]
    {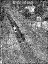}
    \label{fig:hp18}
}

\caption{Video-frame reconstruction under Cauchy noise for two different frames. The first and second rows correspond to frames 1 and 8, respectively. From left to right: original frame, noisy frame, standard PCA reconstruction with $k=1$, and heavy-tailed PCA reconstruction with $k=1$.}
\label{fig:frame1pca}
\label{fig:frame8pca}
\end{figure*}

\bibliographystyle{IEEEtran}
\bibliography{references,references1}

\end{document}